\documentclass{llncs}
\usepackage{amsfonts,amsmath}
\usepackage{xspace}                           
\usepackage{booktabs}                         
\usepackage{graphics}
\usepackage{hyperref}
\usepackage{tikz}
\usepackage{dsfont}

\allowdisplaybreaks

\newcommand{\prob}[1]{\Pr\left(#1\right)}
\newcommand{\expect}[1]{\mathbf{E}\left[#1\right]}

\newcommand{\Real}{\mathbb{R}}
\newcommand{\Natural}{\mathbb{N}}

\newcommand{\sel}{\text{\tt select}\xspace} 
\newcommand{\mut}{\text{\tt mutate}\xspace}

\newcommand{\psel}{\ensuremath{p_\mathrm{sel}}\xspace}
\newcommand{\pmut}{\ensuremath{p_\mathrm{mut}}\xspace}

\DeclareMathOperator{\poly}{poly}

\usepackage{mathrsfs}

\newcommand{\bigO}[1]{\ensuremath{\mathcal{O}\left(#1\right)}}
\newcommand{\bigTheta}[1]{\mathord{\Theta}\mathord{\left(#1\right)}}

\newcommand{\ab}{\hspace{0.125em}}                        
\newcommand{\ie}{\hbox{i.\ab e.}\xspace}                  
\newcommand{\eg}{\hbox{e.\ab g.}\xspace}                  
\newcommand{\cf}{\hbox{c.\ab f.}\xspace}                  
\newcommand{\cwlog}{\hbox{w.\ab l.\ab o.\ab g.}\xspace}   

\usepackage[noend]{algorithm,algorithmic}
\floatname{algorithm}{Algorithm}

\newcommand{\onemax}{\text{\sc OneMax}\xspace} 

\newcommand{\linear}{\text{\sc Linear}\xspace} 

\newcommand{\decomp}{\text{\sc Decomp}\xspace} 

\newcommand{\royalroad}{\text{\sc Rr}\xspace} 

\usepackage{enumitem}

\usepackage{color}

\newcommand{\genbeta}{\beta\xspace}

\begin{document}
\title{Runtime Analysis of 
      Fitness-Proportionate Selection on Linear
      Functions}


\author{
  Duc-Cuong Dang\inst{1},
  Anton Eremeev\inst{2,3},
  Per Kristian Lehre\inst{4}
  }
\authorrunning{D-C. Dang et al.}
\tocauthor{D-C. Dang, A. V. Eremeev, P. K. Lehre}

\institute{
  INESC TEC, Porto, Portugal
\and
  Institute of Scientific Information for Social Sciences RAS, Moscow, Russia\\
\and
  Sobolev Institute of Mathematics, Omsk, Russia\\
\and
  University of Birmingham, Birmingham, United Kingdom\\
\email{duc.c.dang@inesctec.pt, eremeev@ofim.oscsbras.ru,
p.k.lehre@cs.bham.ac.uk}
  }

\maketitle

\begin{abstract}

  This paper extends the runtime analysis of non-elitist evolutionary
  algorithms (EAs) with fitness-proportionate selection from the simple \onemax function
  to the linear functions. Not only does our analysis cover a larger class of fitness
  functions, it also holds for a wider range of mutation rates.
  We show that with overwhelmingly high probability, no linear
  function can be optimised in less than exponential time, assuming
  bitwise mutation rate $\Theta(1/n)$ and population size
  $\lambda=n^k$ for any constant $k> 2$.
  In contrast to this negative result, we also show that for any linear
  function with polynomially bounded weights, the EA achieves
  a polynomial expected runtime if the mutation rate is reduced
  to $\Theta(1/n^2)$ and the population size is sufficiently large.
  Furthermore, the EA
  with mutation rate $\chi/n=\Theta(1/n)$ and modest population
  size $\lambda=\Omega(\ln n)$ optimises the \emph{scaled} fitness function
  $e^{(\chi+\varepsilon)f(x)}$ for any linear function $f$ and
  any $\varepsilon>0$ in expected time $O(n\lambda\ln\lambda+n^2)$.
  These upper bounds also extend to some additively decomposed fitness functions, such as the Royal Road functions.
  We expect that the obtained results may be useful not only for the development of the theory of evolutionary algorithms, but
  also for biological applications, such as the directed evolution.

\end{abstract}

\section{Introduction}\label{sec:intro}

Realising the potential and usefulness of each operator that can constitute
randomised search heuristics (RSH) and their interplay is an important step
towards the efficient design of these algorithms for practical applications.
Theoretical studies, especially runtime analyses of RSH, have
rigorously and successfully contributed to such realisation.
Here and below, by the runtime, or expected optimization time, we mean the expected number of fitness (objective function)
evaluations made until an optimum is found for the first time.
Taking evolutionary algorithms~(EAs) as an example, the proofs
showing how and when the population size, recombination operators,
mixing mutation operators or self-adaptation techniques are
essential can be found in \cite{bib:Dang2017a,bib:Dang2017,bib:Dang2016a,bib:Jansen2001,bib:Lehre2013,bib:Witt2008}.
Moreover, from this type of
studies some better algorithms and operators can also be developed
\cite{bib:Doerr2015a,bib:Doerr2017} and some biologically
meaningful estimates may be obtained~\cite{ES_BGRS18}.

In this paper, we analyse the use of the
\emph{fitness-proportionate} selection in optimising linear
fitness functions and additively decomposed fitness functions. These fitness functions are
among the basic examples of objective functions in mathematical optimization.
Many models in theoretical biology are based on a weak epistasis assumption, \ie genes have approximately additive effect on the genotype fitness, which may be modelled by a linear fitness function.

The fitness-proportionate selection mechanism, also known as
\emph{roulette-wheel} selection, was the main selection used in
the early development of genetic algorithms~(GA) and their
applications \cite{bib:Goldberg1989}. Specifically, the chance of
selecting an individual~$x$ for reproduction is equal to the
fitness~$f(x)$ of $x$ divided by the total fitness of the
population. Thus unlike the rank-based selections (tournament
selection, $(\mu,\lambda)$-selection, ranking selection etc.),
this selection is sensitive to the absolute values of the fitness
function~$f(x)$, and a non-linear scaling of the function may
significantly change its properties. This is often seen as a
weakness from the theoretical view point, but at the same time
there exists a large body of literature reporting applications of
this mechanism in combinatorial optimization (see
e.g.~\cite{AOT97,BCh96,Marchiori1999}), where the proportionate
selection is not necessarily the best practically tested selection
mechanism but at least a competitive one.

Analysis of the proportionate selection is also valuable for the transfer of
methods from the area of evolutionary computation into the biology.
In population genetics, some models of population dynamics account for
fitness-proportionate effect of selection on the genotypes
frequencies, see e.g.~\cite{PAIXAO201528}.
The well-known biotechnological procedure SELEX (Systematic Evolution of Ligands by EXponential enrichment) can be treated as an {\em in vitro} implementation of an EA~\cite{ES_BGRS18}.
SELEX and its variants are valuable tools to identify DNA and RNA
sequences with high affinity for binding a pre-specified target proteins/molecules~\cite{Darmostuk15,Tizei16}.
Such procedures have numerous applications in clinical research, agriculture, metabolic engineering, etc.
The mathematical model of SELEX~\cite{ITG91} shows that the effect of selection in this procedure is the same as the average effect of
the fitness-proportionate selection on a specific population-dependent fitness function 
(see details in Section~\ref{sec:discus}).

The fitness-proportionate selection became popular in evolutionary computation with the seminal
book of Goldberg~\cite{bib:Goldberg1989} and the formalisation of
the so-called simple genetic algorithm~(SGA) for optimisation
problems on bit-strings. The SGA is a {\em non-elitist} EA, \ie
the populations through generations are non-overlapping,
thus the main force guiding the optimisation process is the
fitness-proportionate selection of parents.
Offspring individuals are varied through the recombination of parents,
and through the bitwise mutation operator. The probability $p_{\rm c}$
of applying the recombination (crossover), and the
probability $p_{\rm m}$ of mutation in each bit-position are
tunable parameters. The standard setting
for mutation
is $p_{\rm m}=1/n$. 
Here and below, $n$ is the length of the bit-string.


Interested in SGA without crossover,
Neumann~et~al.~\cite{bib:Neumann2009} concluded that this algorithm
with the standard mutation 
and a population size $\lambda < \log(n)/4$ is inefficient in
optimising any pseudo-Boolean function with a unique optimum.
A shorter proof for this result (but with $\lambda > n^3$) was proposed in~\cite{bib:Lehre2011}.
%

One of the well-known benchmark functions in the theory of evolutionary algorithms is \onemax which counts the number of $1$-bits of the input string $x$. Given appropriate scaling of fitness, the SGA without a crossover was shown to be capable of finding the optimum
of \onemax within expected polynomial time~\cite{bib:Lehre2011}.
A similar conclusion is
made  in~\cite{bib:ErYuJOR17} for the SGA with a constant crossover probability~$p_{\rm c}<1$
optimizing any pseudo-Boolean function
without local optima which are not globally optimal. On
\onemax,
an expected polynomial runtime of the
SGA without crossover was established in~\cite{bib:Dang2016},
assuming a reduction of the mutation
probability to~$1/(6n^2)$.
In~\cite{bib:DK_GECCO2019}, the polynomial runtime bound was significantly reduced.
The results from~\cite{bib:Dang2016}, \cite{bib:DK_GECCO2019} and \cite{bib:Lehre2011} make
use of the so-called \emph{level-based analysis}
technique for proving upper bounds on the expected optimisation time.
The class of linear pseudo-Boolean problems has played a central role
as theoretical benchmark in evolutionary computation, and to develop
suitable analytic techniques.
Droste, Jansen and Wegener \cite{droste_analysis_2002} and He and Yao
\cite{he_drift_2001,he_erratum_2002}
showed independently in 2002 that
the expected optimisation time of the (1+1) EA on linear functions of $n$ variables is
$\Theta(n\log n)$.
Using the drift analysis, Witt proved that the runtime of the (1+1)~EA on linear
functions is $en\ln n+ O(n)$ in expectation and with high probability
\cite{witt_tight_2013} (where $e$ is the basis of the natural logarithm). For mutation probability $p=c/n$ for a
constant $c>0$, the expected optimisation time is
$(1\pm o(1))\frac{e^c}{c}n\ln n$, which is minimised for $p=1/n$ (up
to lower-order terms). Furthermore, he proved that no mutation-based
EA has an expected optimisation time smaller than this (up to
lower-order terms).

Based on the gambler's ruin problem and drift analysis
\cite{bib:Hajek1982}, Happ~et~al.~\cite{bib:Happ2008} showed that
switching the \emph{plus} ``+'' selection for replacement in the
$($1+1$)$~EA and RLS~(Randomized Local Search) to fitness-proportionate selection makes the
algorithms highly inefficient in optimising linear functions.

Several results have also become available for population-based
evolutionary algorithms. Assuming appropriate selective pressure, the
expected optimisation time of the ($\mu$,$\lambda$) EA and many
similar non-elitist evolutionary and genetic algorithms is
$O(n^2+n\lambda\log\lambda)$ \cite{bib:Corus2017}. This bound on all
linear functions is significantly higher than the upper bound
$O(n\lambda)$ which holds for $\onemax$ when $\lambda=\Omega(\log n)$
\cite{bib:Corus2017}. The ($1$+$\lambda$) EA optimises linear
functions in $O(n\log n+n\lambda)$ function evaluations
\cite{doerr_optimizing_2015}. For the $\onemax$, the runtime bound is
$O(n\log n + \lambda n\log\log\lambda/\log\lambda)$.

  The contribution of this paper is twofold. 
  On the negative side, we show that no linear
  function can be optimised in less than exponential time, assuming
  bitwise mutation rate $\Theta(1/n)$ and population size
  $\lambda=n^k$ for any constant $k> 2$.
  On the positive side, we prove that for any linear
  function with polynomially bounded weights, the EA achieves
  a polynomial expected runtime if the mutation rate is reduced
  to $\Theta(1/n^2)$ and $\lambda$ is sufficiently large.
  Furthermore, the EA
  with mutation rate $\chi/n=\Theta(1/n)$ and population
  size $\lambda=\Omega(\ln n)$ optimises the \emph{scaled} fitness function
  $e^{(\chi+\varepsilon)f(x)}$ for any linear function $f$ and
  $\varepsilon>0$ in expected time $O(n\lambda\ln\lambda+n^2)$.

%
The remainder of the paper is
organised as follows. The considered algorithm, tools and proving
techniques are presented in the next section.
Section~\ref{sec:stdsel} presents the general negative result for
the standard setting of fitness-proportionate selection on linear
functions. This is followed by the presentation of different
modifications to the setting and the algorithm so as to make the
mechanisms efficient in Section~\ref{sec:scaling-mrate}. The
possibility to extend the obtained upper bounds to EAs with other
fitness functions, such as the Royal Road function, is considered in
Section~\ref{sec:low_mut_adf}. A discussion of the obtained results and
their potential applicability to biological evolution is given in Section~\ref{sec:discus}.
Conclusions are drawn in Section~\ref{sec:concl}.
Some proofs are excluded from the paper and provided in the appendix.

\section{Preliminaries}\label{sec:prelem}

%
For any $n \in \Natural$, define $[n]:=\{1,2,\dots,n\}$.
The natural logarithm and logarithm to the base $2$ are denoted by
$\ln(\cdot)$ and $\log(\cdot)$ respectively.
For $x \in \{0,1\}^n$, we write $x_i$ for the $i$-th bit value.
The Hamming distance is denoted by $H(\cdot,\cdot)$ and
the Iverson bracket by $[\cdot]$.
Throughout the paper the maximisation of a \emph{fitness function}
$f\colon \mathcal{X} \rightarrow \Real$ over a finite \emph{search
space} $\mathcal{X}$ is considered.
Given a partition of $\mathcal{X}$ into $m$ ordered
subsets/\emph{levels} $(A_1,\dots,A_m)$,
let $A_{\geq j} := \cup_{i=j}^{m}A_i$.
The partition is called \emph{$f$-based} if for all $x \in A_j$
and $y \in A_{j+1}$ it holds that $f(y)>f(x)$ for all $j \in [m-1]$.
A \emph{population} is a vector $P \in \mathcal{X}^{\lambda}$, where
the $i$-th element is called the $i$-th \emph{individual}.
For $A \subseteq \mathcal{X}$, define $|P \cap A|:=|\{i \mid P(i)
\in A\}|$, \ie the count of individuals of $P$ in $A$.
We are
interested in fitness functions on $\mathcal{X}
= \{0,1\}^n$, the so-called \emph{pseudo-Boolean functions},
and their class of linear functions:
$$
  \linear(x) :=\sum_{i=1}^n a_i x_i,
$$
where $a_i \neq 0$ for $i\in [n]$. Due to the symmetry of the mutation
operator that we use (see below), we can assume \cwlog throughout
the paper that the weights are positive and sorted in descending order,
\ie $a_1 \geq a_2 \geq \dots \geq a_n > 0$.

All algorithms considered in this paper fall into the framework of
Algorithm~\ref{algo:EA}. 
Starting with some $P_0$
which is sampled uniformly from $\mathcal{X}^\lambda$, in each
iteration $t$ of the outer loop, a new population $P_{t+1}$ is
generated by independently sampling $\lambda$ individuals from the
existing population $P_t$ using two
operators:
\emph{selection} $\sel\colon \mathcal{X}^\lambda \rightarrow [\lambda]$
and \emph{mutation} $\mut\colon \mathcal{X} \rightarrow \mathcal{X}$.
Here, $\sel$ takes a vector of $\lambda$ individuals as input, then
implicitly makes use of the function $f$, \ie through \emph{fitness
evaluations}, to return the index of the individual to be selected.

\begin{algorithm}
  \caption{Non-Elitist Evolutionary Algorithm \cite{bib:Dang2016,bib:Lehre2011}}\label{algo:EA}
  \begin{algorithmic}[1]
    \REQUIRE ~\\
             Finite state space $\mathcal{X}$, and initial population $P_0 \in \mathcal{X}^\lambda$
    \FOR{$t=0,1,2,\dots$ until termination condition met}
        \FOR{$i=0,1,2,\dots,\lambda$ }
           \STATE Sample $I_t(i):=\sel(P_t)$, and set $x := P_t(I_t(i))$
           \STATE Sample $P_{t+1}(i) := \mut(x)$
        \ENDFOR
    \ENDFOR
  \end{algorithmic}
\end{algorithm}

The function is optimised when an optimum $x^*$, \ie $f(x^*) =
\max_{x \in \mathcal{X}}\{f(x)\}$, appears in $P_t$ for the first
time, \ie $x^*$ is sampled by $\mut$, and the optimisation time
(or runtime) is the number of fitness evaluations made until that
time.

Formally, $\sel$ is represented by a probability distribution over
$[\lambda]$, and we use $\psel(i \mid P)$ to denote the
probability of selecting the $i$-th individual $P(i)$ of $P$. The
\emph{fitness-proportionate selection} is an implementation of
$\sel$ with
$$
  \forall P \in \mathcal{X}^{\lambda}, \forall i \in [\lambda]\colon
    \psel(i \mid P) = \frac{f(P(i))}{\sum_{j=1}^{\lambda} f(P(j))}.
$$
%
We say that $\sel$ is \emph{$f$-monotone} if for all $P \in
\mathcal{X}^\lambda$ and all $i,j \in [\lambda]$ it holds that
$\psel(i \mid P) \geq \psel(j \mid P) \Leftrightarrow f(P(i)) \geq
f(P(j))$. It is easy to see that the fitness-proportionate
selection is $f$-monotone.

We are interested in the following two characteristics of
selection.
The \emph{cumulative selection probability} $\beta$ of $\sel(P)$
for any $\gamma \in (0,1]$ is
$$
  \beta(\gamma,P)
    := \sum_{i=1}^{\lambda} \psel(i\mid P) \cdot \left[ f(P(i)) \geq f_{\lceil \gamma\lambda \rceil}
    \right],\ \text{where} \ P \in \mathcal{X}^{\lambda},
$$
assuming a sorting $(f_1,\cdots,f_\lambda)$ of the fitnesses of $P$
in descending order. In essence, $\beta(\gamma, P)$ is the
probability of selecting an individual at least as good as the
$\lceil \gamma\lambda\rceil$-ranked individual of $P$,
When sampling $\lambda$ times with $\sel(P_t)$ and recording the outcomes as
vector $I_t \in [\lambda]^\lambda$, the \emph{reproductive rate} of $P_t(i)$ is
$$
  \alpha_t(i): = \expect{R_t(i) \mid P_t}
    \text{ where } R_t(i) := \sum_{j=1}^{\lambda}[I_t(j)=i].
$$
Thus $\alpha_t(i)$ is the expected number of times that
$P(i)$ is selected. The reproductive rate $\alpha_0$ of
Algorithm~\ref{algo:EA} 
is defined as $\alpha_0 := \sup_{t\geq 0} \max_{i\in[\lambda]}\{
\alpha_t(i) \}$.

The mutation operator $\mut$ is represented by a transition matrix
$\pmut\colon \mathcal{X}\times \mathcal{X} \rightarrow [0,1]$, and
we use $\pmut(y \mid x)$ to denote the probability to mutate an
individual $x$ into $y$. On $\mathcal{X} = \{0,1\}^n$, the
\emph{bitwise mutation} with mutation rate (probability) $\chi/n$
is an implementation of $\mut$ that satisfies
$$
  \forall x, y \in \{0,1\}^n\colon
  \pmut(y \mid x)
    = \left(\frac{\chi}{n}\right)^{H(x,y)}\left(1 - \frac{\chi}{n}\right)^{n - H(x,y)}.
$$
Note that the bitwise mutation treats the bit values $0$ and $1$
indifferently, and so for the bit positions. This allows the earlier
mentioned assumption on the positiveness and on the sorting of the
weights for $\linear$ functions.

To bound the expected optimisation time of Algorithm~\ref{algo:EA}
from above, we will use the \emph{level-based analysis}~\cite{bib:Corus2017}. The following theorem is taken from
Corollary~7 in \cite{bib:Corus2017} and tailored to the setting of
an $f$-based partition and $p_\mathrm{c}=0$. 
Thus it fits Algorithm~\ref{algo:EA}, and
is an improvement to Theorem~8 of~\cite{bib:Dang2016}.
\begin{theorem}\label{thm:level-based-theorem}
Given an $f$-based partition $(A_1,\dots,$ $A_{m})$ of $\mathcal{X}$, let
$P_t\in\mathcal{X}^\lambda$ be the population of
Algorithm~\ref{algo:EA} in generation $t,$
$t\in\mathbb{N},$ and define $T:=\min\{t\lambda \mid |P_t \cap A_{m}| > 0\}$.
If 
there exist 
  $s_1,\dots,s_{m-1},p_0,\delta \in(0,1]$,
$\gamma_0 \in (0,1)$ such that
%
  \begin{description}[noitemsep]
  \item[(M1)] $\forall P\in\mathcal{X}^\lambda, \forall j\in [m-1] \colon$
    $\displaystyle
       \pmut\left( y\in A_{\geq j+1} \mid x \in A_j \right)\geq s_j,$
  \item[(M2)] $\forall P\in\mathcal{X}^\lambda, \forall j\in [m-1] \colon$
    $\displaystyle
     \pmut\left( y\in A_{\geq j} \mid x\in A_j \right)\geq p_0,$
  \item[(M3)] $\forall P \in \left(\mathcal{X}\setminus A_{m}\right)^\lambda,
               \forall \gamma\in(0,\gamma_0]\colon $
  $\displaystyle
               \beta(\gamma, P) \geq (1+\delta)\gamma/p_0,$ 
  \item[(M4)] population size 
  $\displaystyle \lambda \geq
               \frac{4}{\gamma_0\delta^2} \ln\left(\frac{128 m}{\gamma_0s_*\delta^2}\right),
               \text{ where } s_*:=\min_{j\in[m-1]} \{s_j\},
  $
  \end{description}
  then
  $
  \expect{T}
   <
    \left(\frac{8}{\delta^{2}}\right)
    \sum_{j=1}^{m-1}\left(\lambda
    \ln\left(\frac{6\delta\lambda}{4+\gamma_0
    s_j\delta\lambda}\right)+\frac{1}{\gamma_0 s_j}\right).
  $
\end{theorem}

As an alternative to Theorem~\ref{thm:level-based-theorem} we use the
new level-based theorem based on the {\em multiplicative up-drift}~\cite{bib:DK_GECCO2019}.
Theorem~3.2 from~\cite{bib:DK_GECCO2019} implies the following:
\begin{theorem}\label{thm:level-based-theorem1}
Given an $f$-based partition $(A_1,\dots,$ $A_{m})$ of $\mathcal{X}$,
define ${T:=\min\{t\lambda \mid |P_t \cap A_{m}| > 0\}}$ where for all
$t\in\mathbb{N}$, $P_t\in\mathcal{X}^\lambda$ is the population of
Algorithm~\ref{algo:EA}.
If 
there exist 
  $s_1,\dots,s_{m-1},p_0,\delta \in(0,1]$,
%
$\gamma_0 \in (0,1)$, such that conditions (M1)--(M3) of Theorem~\ref{thm:level-based-theorem} hold and
%
  \begin{description}[noitemsep]
  \item[(M4')] for some constant $C>0$, the population size $\lambda$ satisfies
  $$
  \lambda \geq
               \frac{8}{\gamma_0\delta^2} \log\left(\frac{C m}{\delta} \left(\log \lambda +\frac{1}{\gamma_0 s_* \lambda}\right) \right),
               \text{ where } s_*:=\min_{j\in[m-1]} \{s_j\},
  $$
  \end{description}
  then
  $
  \expect{T}
   =\mathcal{O}
    \left(\frac{m\lambda \log(\gamma_0 \lambda)}{\delta} +
    \frac{1}{\delta}
    \sum_{j=1}^{m-1}\frac{1}{\gamma_0 s_j}\right).
  $
\end{theorem}

Theorem~\ref{thm:level-based-theorem1} improves on Theorem~\ref{thm:level-based-theorem} in terms of dependence on $\delta$,
but only gives an asymptotical bound.
Its proof outline is analogous to that of Theorem~\ref{thm:level-based-theorem}.

Our lower bound is based on the \emph{negative drift theorem
  for populations}~\cite{bib:Lehre2010}.
%
\begin{theorem}\label{thm:negative-drift-pop}\label{th:negative_drift}
Consider Algorithm~\ref{algo:EA} on $\mathcal{X} = \{0,1\}^n$ with
bitwise mutation rate $\chi/n$ and population size $\lambda =
\poly(n)$, let $a(n)$ and $b(n)$ be positive integers such that
$b(n)\leq n/\chi$ and $d(n) = b(n) - a(n) = \omega(\ln n)$. Given
$x^* \in \{0,1\}^n$, define $T(n) := \min\{t \mid |P_t \cap \{x
\in \mathcal{X} \mid H(x,x^*) \leq a(n)\}| > 0\}$. If there exist
constants $\alpha>1$, $\delta>0$ such that
  \begin{description}[noitemsep]
  \item[(1)] $\forall t\geq 0$, $\forall i \in [\lambda]\colon$
               if $a(n) < H(P_t(i),x^*) < b(n)$ then $\alpha_t(i) \leq \alpha$,
  \item[(2)] $\displaystyle
               \psi := \ln(\alpha)/\chi + \delta < 1$,
  \item[(3)] $\displaystyle
               b(n)/n < \min\left\{1/5,
                                           1/2 - \sqrt{\psi(2-\psi)/4}\right\}$,
  \end{description}
  then $\prob{T(n)\leq e^{cd(n)}} = e^{-\Omega(d(n))}$ for some constant $c>0$.
\end{theorem}

%

\section{Fitness-Proportionate Selection with Standard Mutation Rates is Inefficient}\label{sec:stdsel}
%
In this section, we consider Algorithm~\ref{algo:EA} with
fitness-proportionate selection and standard bitwise mutation
given a constant value of the parameter~$\chi>\ln 2$. This
algorithm turns out to be inefficient on the whole class of linear
fitness functions. For the proof we will use the same approach as
suggested for lower bounding the EA runtime on the \onemax fitness
function in~\cite{bib:Lehre2011}. In order to obtain an upper
bound on the reproductive rate,
we first show that, roughly speaking, it is
unlikely that the average fitness of the EA population becomes
less than half the optimal sometime during an exponential number
of iterations.

\begin{lemma}\label{lemma:antifitness_sum}
  Let $\varepsilon>0$ and $\delta>0$ be constants and let $f(x)= \linear(x)$ with $f^*:=\sum_{i=1}^n a_i$.
  Define $T$ to be the smallest
  $t$ such that Algorithm~\ref{algo:EA} using an $f$-monotone selection mechanism,
  bitwise mutation with $\chi=\Omega(1),$
  and population size $\lambda\geq n^{2+\delta},$ has a population
  $P_t$ where
  $
  \sum_{j=1}^\lambda f(P_t(j))\leq \lambda (f^*/2)(1-\varepsilon).
  $
  Then there exists a constant $c>0$ such that ${\prob{T\leq e^{cn}} = e^{-\Omega(n^{\delta})}.}$
\end{lemma}

The proof of Lemma~\ref{lemma:antifitness_sum} is analogous to that of Lemma~9 from~\cite{bib:Lehre2010}.

The following theorem establishes a lower bound for the expected
runtime and for approximation to the optimum in terms of distance
in solution space, using the negative drift theorem for
populations~\cite{bib:Lehre2010} (see
Theorem~\ref{th:negative_drift}) and Lemma~\ref{lemma:antifitness_sum}.

\begin{theorem}\label{theorem:linear_approximation}
 Let $\delta>0$  be a constant,
 $f(x)=\linear(x)$, $f^*:=\sum_{i=1}^n a_i$,
 then there exists a constant $c>0$ such that during $e^{cn}$ generations
 Algorithm~\ref{algo:EA} with population size~$\lambda\ge
 n^{2+\delta},$ and $\lambda=\poly(n),$ bitwise mutation rate
 $\chi/n$ for any constant $\chi>\ln(2)$, and
 fitness-proportionate selection, with probability at most~$\lambda e^{-\Omega(n^{\delta})}$

 (i) obtains the optimum of~$f$,

 (ii) obtains a search point with less than $\frac{n(1-\varepsilon)}{2} \cdot
  \left(1-\sqrt{\frac{\ln 2}{2 \chi} - \left(\frac{\ln 2}{2 \chi}\right)^2 +\frac{3}{4}}\right)$
  zero-bits for any constant $\varepsilon\in(0,1)$.
\end{theorem}


\begin{proof}

It follows by Lemma~\ref{lemma:antifitness_sum} that with
probability at least~$1-e^{-\Omega(n^{\delta})}$ for any constant
$\varepsilon'>0$ we have $\sum_{j=1}^{\lambda} f(P_t(j)) \ge
\lambda (f^*/2)(1-\varepsilon')$ during $e^{c'n}$ iterations for
some constant $c'>0$. Otherwise, with probability
$e^{-\Omega(n^{\delta})}$ we can pessimistically assume that the
optimum is found before iteration~$e^{c'n}$.

With probability at least~$1-e^{-\Omega(n^{\delta})}$ the
reproductive rate~$\alpha_0$ satisfies
\begin{equation}\label{eqn:alpha0ub}
\alpha_0\le \frac{\lambda f^*}{\lambda (f^*/2)
(1-\varepsilon')}=\frac{2}{1-\varepsilon'}=:\alpha.
\end{equation}

Part (i). Inequality~(\ref{eqn:alpha0ub}) implies that for a
sufficiently small~$\varepsilon'$ holds $\alpha_0<e^{\chi}$ and
analogously to Corollary~1 from~\cite{bib:Lehre2010}, we prove
that the probability to optimise a linear function~$f$ with a
single optimum within $e^{c''n}$ generations is $\lambda
e^{-\Omega(n)}$ for some constant $c''>0$. The linear function~$f$
has a single optimum because all $a_i>0$. Therefore with
$c=\min\{c',c''\},$ part~(i) of the theorem holds.

Part (ii). For any $\varepsilon'>0,$ the upper bound~$\alpha$ from
inequality~(\ref{eqn:alpha0ub}) satisfies condition~1 of
Theorem~\ref{th:negative_drift} for any $a(n)$ and $b(n)$. Note
that the upper bound~$\alpha$ from~(\ref{eqn:alpha0ub}) also
satisfies the inequality
$\ln(\alpha)=\ln(2)-\ln(1-\varepsilon')<\ln(2)+\varepsilon' e$ for
any $\varepsilon'\in(0,1/e)$.

Condition~2 of Theorem~\ref{th:negative_drift} requires that
${\ln(\alpha)/\chi+\delta'<1}$ for a constant ${\delta'>0}$. This
condition is satisfied because ${\frac{\ln(\alpha)}{\chi}<
\frac{\ln(2)+\varepsilon'e}{\chi}<1}$ for a sufficiently small
$\varepsilon'.$ Here we use the assumption that $\chi>\ln(2)$ from
the formulation of part~(ii). It suffices to assume
$\varepsilon'=\frac{\chi-\ln 2}{2e}.$ Define
$\psi:=\frac{\ln(2)+\varepsilon'e}{\chi}=\frac{\ln(2)}{2\chi}+\frac{1}{2}.$

To ensure Condition~3 of Theorem~\ref{th:negative_drift}, we
denote $r:=\ln(2)/\chi<1$ and
 $$
 M(\chi):=\frac{1-\sqrt{\psi(2-\psi)}}{2}=
 \frac{1-\sqrt{r/2-r^2/4+3/4}}{2}.
 $$
 Note that
 $M(\chi)$ is decreasing in~$r$ and therefore increasing
 in~$\chi$, besides that
$M(\chi)$ is independent of~$n$ and of coefficients~$a_i$. Now we
define $a(n)$ and $b(n)$ so that $b(n)<M(\chi)n$ and
$b(n)-a(n)=\omega(n)$. Assume that $a(n):=n (1-\varepsilon)
M(\chi)$ and $b(n):=n(1-\varepsilon/2) M(\chi)$, where
$\varepsilon>0$ is a constant given in the formulation of
part~(ii). Application of Theorem~\ref{th:negative_drift}
completes the proof. \qed
\end{proof}




Now suppose that all coefficients of $\linear$ differ at most by a
factor of~$r$, i.e. for all~$i,j$ holds $a_i\le r a_j$. W.l.o.g.
we can assume that the coefficients $a_i$ are non-increasing
in~$i$. Let $x^*= (1,\dots,1)$ denote the optimum of $\linear.$
Then any solution~$x$ such that $H(x,x^*)\ge d>0$ will have a
fitness $f(x)\le f^*-\sum_{i=n-d+1}^n a_i\le (1-d/(rn))f^*$.
Therefore, claim~(ii) of
Theorem~\ref{theorem:linear_approximation} implies the following
inapproximability result in terms of fitness function.

\begin{corollary}\label{cor:linear_approximation}
 Let $\delta>0$  be a constant,
 $f(x)= \linear(x)$ with maximal value $f^*:=\sum_{i=1}^n a_i$, such that for all $i,j$ holds $a_i\le r
 a_j,$ and assume population size~$\lambda\geq n^{2+\delta}.$
 Then
 there exists a constant $c>0$, such that during the first $e^{cn}$ generations,
 with probability at least~$1-\lambda e^{-\Omega(n^{\delta})}$,
 Algorithm~\ref{algo:EA} using
 fitness-proportionate selection and bitwise mutation rate $\chi/n$
 for any constant $\chi>\ln(2)$, does not
 obtain a solution~$x$ with an approximation factor
$$
\frac{f(x)}{f^*} \le 1-\frac{1}{2r} \cdot
  \left(1-\sqrt{\frac{\ln 2}{2 \chi} - \left(\frac{\ln 2}{2 \chi}\right)^2
  +\frac{3}{4}}\right).
$$

\end{corollary}


\section{Fitness-Proportionate Selection with Low Mutation or Fitness Scaling is Efficient}\label{sec:scaling-mrate}

Early experimental studies of SGA suggested setting the mutation
rates to be inversely proportional to the population size
(\eg see \cite{bib:Goldberg1989}). Rigorous runtime analyses
of EAs, starting with the simplest algorithm (\eg see \cite{droste_analysis_2002}),
have made mutation rates inversely proportional to the
problem dimension, \ie $1/n$ or more generally $\chi/n$ for
some constant $\chi$, the standard setting for mutation.
As seen in the previous section, the setting is in fact
detrimental for fitness-proportionate selection on linear functions,
and the result agrees with the previous studies of SGA
\cite{bib:Oliveto2011,bib:Oliveto2014,bib:Oliveto2015} on \onemax.
However, \cite{bib:Dang2016} made an important discovery that
turning down the mutation rate to $1/(6n^2)$ brings the expected
runtime on \onemax back to the polynomial domain. In this section, we
first generalise the result of \cite{bib:Dang2016} to linear
functions with not too large weights. We then show that polynomial
expected runtime can also be achieved if the fitness is
exponentially scaled.
%
Throughout the section, we suppose that all weights~$a_j$ are integer.


\begin{theorem}\label{thm:GA-on-pseudo-boolean-func}
The expected runtime of the Algorithm~\ref{algo:EA} on $\linear$ where
$a_1$ is the largest weight, using
\begin{itemize}
\item fitness-proportionate selection,
\item bitwise mutation with mutation rate $\chi/n$ where
$\chi=(1-c)/(n a_1)$
for any constant $c \in (0,1)$
\item population size
$\lambda \ge 2^8 n^2 a_1^2 c^{-3} \left( \ln \left(\frac{(n+1)^5 a_1^3}{c(1-c)}\right)+ 11 \right),$

\end{itemize}
is no more than
$\frac{2^7 n^3 a_1^2}{c^2} \left( \lambda \ln(3\delta\lambda/2) + \frac{4 e n^2 a_1}{c(1-c)}\right).$
\end{theorem}


\begin{proof} 
The proof applies Theorem~\ref{thm:level-based-theorem}
using the partition:
$A_n:=\{1^n\}$,
$A_{j} := \left\{x \mid \sum_{i=1}^j a_i \leq \linear(x) < \sum_{i=1}^{j+1} a_i \right\}$
for $j \in \{0\} \cup [n-1]$, and here $m=n+1$.

The partition is such that given $x\in A_{j}$ for any~$j<n$,
among the first $j+1$ bits, there must be at least one $0$-bit,
thus it suffices to flip the left most $0$-bit while keeping
all the other bits unchanged to produce a search point at a higher
level. The probability of such an event is
$ \frac{\chi}{n}\left(1 - \frac{\chi}{n}\right)^{n-1}
  >    \frac{\chi}{n}\left(1 - \frac{1}{n}\right)^{n-1}
  \geq \frac{1-c}{en^2 a_1} =: s_j = s_*
$, and this choice of $s_j$ satisfies~(M1).
%
To satisfy (M2), we pick $p_0 := (1-\chi/n)^n$, \ie the probability
of not flipping any bit position by mutation.

In (M3), we choose $\gamma_0 := c/4$ and for any $\gamma \leq
\gamma_0$, let $f_\gamma$ be the fitness of the $\lceil\gamma
\lambda\rceil$-ranked individual of any given $P \in
\mathcal{X}^\lambda$. Thus there are at least $k \geq
\lceil\gamma\lambda\rceil \geq \gamma\lambda$ individuals with
fitness at least $f_\gamma$ and let $s \geq k f_\gamma \geq
\gamma\lambda f_\gamma$ be their sum of fitness. Since the
weights~$a_i$ are all integers, we can pessimistically assume that
individuals with fitness less than $f_\gamma$ have fitness
$f_\gamma - 1$, therefore 
\begin{align*}
  \beta(\gamma, P)
    &\geq \frac{s}{s + (\lambda - k)(f_\gamma - 1)}
     \geq \frac{s}{s + (\lambda - \gamma\lambda)(f_\gamma - 1)} \\ 
    &\geq \frac{\gamma\lambda f_\gamma}{\gamma\lambda f_\gamma + (\lambda - \gamma\lambda)(f_\gamma - 1)}
     =    \frac{\gamma}{1 - (1 - \gamma)/f_\gamma} \\ 
    &\geq \frac{\gamma}{1 - (1 - c/4)/f^*}
     \geq \gamma e^{(1 - c/4)/f^*}, 
\end{align*}
where $f^* := \sum_{i=1}^{n} a_i$ and in the last line we apply
the inequality $e^{-x}\ge 1-x.$
Note that $p_0 = (1-\chi/n)^{n} \geq e^{-\chi/(1-\varepsilon)}$
for any constant $\varepsilon \in (0,1)$ and sufficiently large
$n$. 
Indeed, by Taylor theorem, $e^{-z}=1-z+z\alpha(z),$ where
$\alpha(z)\to 0$ as $z\to 0$. So given any $\varepsilon>0$, for
all sufficiently small~$z>0$ holds $e^{-z}\le 1-(1-\varepsilon)z$.
For any~$\varepsilon\in(0,1)$ we can assume that
$z=\chi/(n(1-\varepsilon)),$ then for all sufficiently large~$n$
it holds that $(1-\chi/n)^{n}\ge
e^{-zn}=e^{-\chi/(1-\varepsilon)}.$
So we conclude that
\begin{equation*} 
  \beta(\gamma, P) p_0
    \ge \gamma e^{(1 - c/4)/f^*} e^{-\chi/(1 - \varepsilon)}
    \ge \gamma\left(1+\frac{1 - c/4 -\chi f^*/(1-\varepsilon)}{f^*}\right).
\end{equation*}
Since $\chi f^* \leq \chi n a_1 = 1 - c$, choosing
$\varepsilon := 1 - \frac{1-c}{1-c/2} \in (0,1)$
implies $\chi f^*/(1 - \varepsilon) \leq 1 - c/2$. Condition (M3)
then holds for $\delta:=c/(4 n a_1)$ because
\begin{equation*} 
\beta(\gamma, P) p_0
  \geq \gamma\left(1+\frac{1 - c/4 - (1 - c/2)}{f^*}\right)
  \geq \gamma\left(1+\frac{c}{4 n a_1}\right).
\end{equation*}

Condition~(M4) requires that the population size must be at least
\begin{align*}
\frac{4}{\gamma_0 \delta^2} \ln \left( \frac{128 m}{\gamma_0 s_* \delta^2 }\right)
   &=    \frac{4 }{ (c/4)(c/(4n a_1))^2 } \ln\left(
                   \frac{128(n+1)}{ (c/4)((1-c)/(e n^2 a_1))(c/(4n a_1))^2}\right) \\
   &<    \frac{2^8 n^2 a_1^2}{c^3} \left(\ln(n+1) + 4 \ln{n} + 3 \ln{a_1} + 11 + \ln{\frac{1}{c(1-c)}}\right)
\end{align*}
Which holds by the assumption on $\lambda$.
Theorem~\ref{thm:level-based-theorem} now
implies
\begin{align*}
\expect{T}
    &\leq \frac{8}{\delta^2} \sum_{j=1}^{n} \left(\lambda \ln(3\delta\lambda/2) + \frac{1}{\gamma_0 s_j}\right)
     = \frac{2^7 n^2 a_1^2}{c^2} \left( n \lambda \ln(3\delta\lambda/2) + \frac{4 e n^3 a_1}{c(1-c)}\right).\qed
\end{align*}

\end{proof}

When $\lambda$ is large enough, the runtime bound of Theorem~\ref{thm:GA-on-pseudo-boolean-func} is in the order of $\mathcal{O} \left(n^3 a_1^2 \lambda \ln{\lambda} \right)$, which can be asymptotically improved using Theorem~\ref{thm:level-based-theorem1} instead of Theorem~\ref{thm:level-based-theorem} as follows.

\begin{theorem}\label{thm:GA-on-pseudo-boolean-func1}
The expected runtime of the Algorithm~\ref{algo:EA} on $\linear,$
using
\begin{itemize}
\item fitness-proportionate selection,
\item bitwise mutation with mutation rate $\chi/n$ where
$\chi=(1-c)/(n a_1)$
for any constant $c \in (0,1)$
\item population size
$\lambda \ge c' n^2 a_1^2 \ln(n a_1),$  $\lambda={\cal O}\left((na_1)^K\right),$
where  $c'$ and $K$ are positive sufficiently large constants,
\end{itemize}
is
$\bigO{ n^2 {a_1} \lambda \log (na_1) + {n^3 a_1^2}}$.
\end{theorem}

\begin{proof}
The proof differs from that of Theorem~\ref{thm:GA-on-pseudo-boolean-func} only in verification of the last condition.
To verify condition~(M4'), we assume $C=1$ and note that
$$
               \frac{8}{\gamma_0\delta^2} \log\left(\frac{C m}{\delta} \left(\log \lambda +\frac{1}{\gamma_0 s_* \lambda}\right) \right)
$$
$$
={\cal O}\left(n^2 a_1^2 \log\left({\cal O}(n^2 a_1) \left(\log(n^K) +\frac{e n^2 a_1}{(c/4)(1-c)\lambda} \right)\right)\right)
$$
$$
={\cal O}(n^2 a_1^2\log(n a_1)),
$$
so~(M4) holds if $c'$ is large enough.
By Theorem~\ref{thm:level-based-theorem},
$\expect{T}={\mathcal O}(n^2 a_1 \lambda \log \lambda + n^3 a_1^2)={\mathcal O}(n^2 a_1 \lambda \log(na_1)+ n^3 a_1^2)$.\qed
\end{proof}

In the case of $\onemax$ where $a_1=1$, the application of the
Theorem~\ref{thm:GA-on-pseudo-boolean-func1} for $\lambda = \bigTheta{n^2 \ln{n}}$ gives
$
  \expect{T} = \bigO{n^{4} \log^2{n}},
$
the same as the upper bound in Theorem~4.1 in~\cite{bib:DK_GECCO2019}, which is generalized here.
Note that Theorems~\ref{thm:GA-on-pseudo-boolean-func} and~\ref{thm:GA-on-pseudo-boolean-func1} give
expected polynomial bounds only if $a_1$ is polynomially bounded.



In accordance with~\cite{bib:Neumann2009}, we call {\em
exponential fitness scaling} the following modification of the
original fitness function~$f(x)$, given a tunable parameter~$c>0$:
\begin{equation}\label{eqn:scale}
f(x,c):=c^{f(x)}.
\end{equation}
In~\cite{bib:Neumann2009}, the EA was shown to optimise
efficiently $\onemax$,
assuming $c:=\lambda-1,$ which grows as~$\Omega(\log n)$. In our study $c$
is assumed to be a constant.

\begin{theorem}\label{thm:GA-on-scaled_linear}
If Algorithm~\ref{algo:EA} is using
\begin{itemize}
\item fitness-proportionate selection with fitness~$c^{\linear(x)}$,
$c>e^{\chi}$,
\item bitwise mutation with mutation rate $\chi/n$ for
a constant $\chi>0$,
\item population size $\lambda \geq \frac{4c}{\varepsilon^3}
\ln \left( \frac{128 (n+1)^2ce}{\varepsilon^3\chi} \right),$ where the constant $\varepsilon:=\sqrt[3]{\frac{c}{e^{\chi}}}-1,$

\end{itemize}
then its expected runtime on $\linear$ is
no more than
$$
\frac{8}{\varepsilon^2} \left(\lambda n \ln(3\varepsilon\lambda/2) + \frac{n^2 e c}{\varepsilon \chi}\right) = \bigO{n \lambda \ln{\lambda} + n^2}.
$$
\end{theorem}

The proof arguments
are analogous to those of Theorem~\ref{thm:GA-on-pseudo-boolean-func}.
The difference here is that when applying
Theorem~\ref{thm:level-based-theorem} the parameter $\delta$
of (M3) can be set to a constant $\varepsilon$ as for the rank-based selection~\cite{bib:Corus2017}.
For this reason we also do not consider use of Theorem~\ref{thm:level-based-theorem1}
which does not give much benefit when $\delta$ is a constant.






\section{Fitness-Proportionate Selection and Separable Additively Decomposed Functions}\label{sec:low_mut_adf}

The result of Theorem~\ref{thm:GA-on-pseudo-boolean-func} may be
extended from linear functions to the class of separable
additively decomposed functions, where the elementary functions
are Boolean and substrings are non-overlapping, so that they partition the string (see
e.g.~\cite{MMR}). Let $f_{\ell}(x)\in\{0,1\}$ be the Boolean
function defined by the bits of the substring~$\sigma_\ell$ of the
string~$x$, where $\ell \in [N]$ and $N$ is the number of
substrings. We assume that the separable additively decomposed
fitness function is given by
$\decomp(x):=\sum_{\ell=1}^N a_{\ell} f_{\ell}(x)$,
where all $a_\ell$ are non-zero integers.
Similar to $\linear$, we can assume \cwlog $a_1 \geq a_2 \geq
\dots \geq a_N > 0$. Furthermore let $r$ be the maximum number
of bits involved in any subsequence $\sigma_{\ell}$, then
we have



\begin{theorem}\label{thm:GA-on-adf}
The expected runtime of Algorithm~\ref{algo:EA} on
$\decomp$ with the largest weight $a_1$ and the
longest subsequence length $r$
using
\begin{itemize}
\item fitness-proportionate selection,
\item bitwise mutation with mutation rate $\chi/n$ where
$\chi=(1-c)/(n a_1)$ for any constant $c \in (0,1)$,
\item population size
$\lambda \ge c' n^2 a_1^2 r \ln({n}{a_1}), \lambda= \bigO{(na_1)^K}$ for sufficiently large
positive constants $c'$ and $K$,
\end{itemize}
is no more than $\bigO{n^2 a_1 \lambda \log(n a_1) + n^{2r+2} a_1^{r+1}(1-c)^{-r}}$.

\end{theorem}

The proof is analogous to that of
Theorem~\ref{thm:GA-on-pseudo-boolean-func1}. In the case of fitness scaling, a result similar to Theorem~\ref{thm:GA-on-scaled_linear}
may be obtained for the \decomp function as well.

One of the well-known examples of the \decomp function is
the Royal Road function~\cite{NCM99}. In what follows, we use the special case of Royal Road functions defined in~\cite{SW04}.
Here all substrings have
equal length $r$ and $N=n/r$:
$$
\royalroad_r(x)
    := \sum_{i=0}^{n/r-1} a_i \prod_{j=1}^{r} x_{ir+j}.
$$
The behavior of Algorithm~\ref{algo:EA} on the fitness functions
$\royalroad_r(x)$ may be considered as a simplified model of population
dynamics in the presence of neutrality regions in a biological
fitness landscape~\cite{NCM99}.
Application of Theorem~\ref{thm:GA-on-adf} implies that in the case of $r=\bigO{1}$
and unit weights~$a_i$, Algorithm~\ref{algo:EA} with
fitness-proportionate selection and bitwise mutation has a
polynomial runtime, given appropriate parameters~$\lambda$
and~$\chi$.

\section{Discussion of Possible Transfer of Results to Biology}\label{sec:discus}

We expect that the results obtained here may be useful not only in development of the theory of evolutionary algorithms, but
also in biological applications, e.g. to estimate the chances for success in directed evolution~\cite{Tizei16}.
Some theoretical bounds from the theory of EAs have been applied in analysis of SELEX procedure for gene promoters in~\cite{ES_BGRS18}, assuming the $(\mu,\lambda)$-selection.
However, more detailed mathematical models of the selection process in SELEX~\cite{ITG91} show that the fraction of any genotype~$i$ in the next generation is proportional to $F_{it}/(Kd_i+C_t(P))$ (\cf Equation (5) in~\cite{ITG91}), where $F_{it}$ is the fraction of genotype~$i$ in the current generation~$t$, $Kd_i$ is the dissociation constant of the genotype~$i,$  and the value~$C_t(P)$ depends on the current population~$P$ and tunable parameters of the SELEX procedure, but it does not depend on~$i$. Therefore, the effect of such process is the same as the expected result of the fitness-proportionate selection on fitness function $1/(Kd_i+C_t(P))$ and the EAs with fitness-proportionate selection is more relevant for SELEX modelling than the EAs with a rank-based selection. The computational experiments~\cite{ES_BGRS18} have shown that the upper bounds from~\cite{bib:Corus2017} are not 
tight for practical applications and the same can be expected in the case of fitness-proportionate selection. Thus further theoretical research is required.

Many models in theoretical biology are based on the assumption that there is almost no epistasis and genes have approximately additive effect on the genotype fitness, which may be modelled by a fitness function \linear or \decomp.
However sometimes it is more appropriate to assume multiplicative effects of genes~\cite{bib:MILLER17}.
The counterparts of such biological models can be found in the EAs
using proportionate selection {\em with scaling,} applied to functions from \linear or \decomp. It might be necessary
to account for some features, ignored in EAs, to make the theory of EAs meaningful for such models
in population biology but the tools developed in the theory of EAs seem to be flexible enough for that purpose.

\section{Conclusions} \label{sec:concl}
The paper extends runtime analysis of fitness-proportionate
selection, from the \onemax function, to the class of linear fitness
functions. Not only does our analysis hold for a larger class of
problems than before, the ranges of parameters involved, such as the mutation rate
and the fitness scaling factor, are also significantly extended. The
improved results follow from the application of the
new level-based theorems and a more detailed analysis of the constants
involved.

On the negative side, we show that non-elitist EAs with fitness-proportionate selection and
standard bitwise mutation, given a constant parameter $\chi > \ln 2$, is inefficient on
the whole class of linear fitness functions. On the positive side, we prove that the runtime can be
turned to polynomial on any linear function with moderate weights by means of reduced mutation rate or by scaling of the fitness. These results are extended to the additively decomposed fitness functions, which 
can be seen 
as simplified fitness landscapes in biology.


\section*{Acknowledgements}
The travel expenses were supported by Ramsay Fund.

\bibliographystyle{splncs03}
\bibliography{references}


\newpage
\appendix
\section*{\appendixname}

The following proofs were omitted from the main part of the paper, and
have been included here for the benefit of the reviewers.

\begin{proof}[of Theorem~\ref{thm:level-based-theorem1}]
We show that conditions (M1-4) imply those of (G1-3) in
Theorem~3.2 from~\cite{bib:DK_GECCO2019}. The proof is analogous to that of Corollary~7 from \cite{bib:Corus2017}.

Let us start with~(G2). Assume that $|P \cap A_{\geq j}| \geq \gamma_0 \lambda$ and
$|P \cap A_{\geq j+1}|\geq \gamma\lambda > 0$ for some $\gamma
\leq \gamma_0$. To create an individual in $A_{\geq j+1}$, it
suffices to pick an $x \in |P \cap A_k|$ for any $k \geq j+1$ and
mutate it to an individual in $A_{\geq k}$, the probability of
such an event, according to (M2) and (M3), is at least $\genbeta(\gamma,P) p_0  \geq (1+\delta)\gamma$. So
(G2) holds.
%

We are given $|P \cap A_{j}| \geq \gamma_0 \lambda$.
Thus, with probability $\genbeta(\gamma_0,P)$, the selection mechanism
chooses an individual $x$ in either $A_j$ or $A_{\geq j+1}$. If
$x\in A_j$, then the mutation operator will by~(M1) upgrade~$x$ to $A_{\geq
j+1}$ with probability $s_j$. If $x\in A_{\geq j+1}$, then by
(M2), the mutation operator leaves the individual in $A_{\geq
j+1}$ with probability $p_0$. So the probability
of producing an individual in $A_{\geq j+1}$ is 
at least
$
\genbeta(\gamma_0,P) \min\{s_j,p_0\}
   \geq \genbeta(\gamma_0,P) s_j p_0
   > \gamma_0 s_j
$
and (G1) holds with $z_j=\gamma_0 s_j,$ $z_*=\gamma_0 s_*$.

Given $z_*=\gamma_0 s_*$,
condition~(M4) yields~(G3).

Conditions (G1--3) are satisfied and Theorem~3.2 from~\cite{bib:DK_GECCO2019} gives
\begin{align*}
  \expect{T}
 =\mathcal{O}
    \left(\frac{m\lambda \log(\gamma_0 \lambda)}{\delta} +
    \frac{1}{\delta}
    \sum_{j=1}^{m-1}\frac{1}{z_j}\right)
 =
\mathcal{O}
    \left(\frac{m\lambda \log(\gamma_0 \lambda)}{\delta} +
    \frac{1}{\delta}
    \sum_{j=1}^{m-1}\frac{1}{\gamma_0 s_j}\right). \
    \qed
\end{align*}
\end{proof}

\begin{proof} [of Lemma~\ref{lemma:antifitness_sum}]
  For the initial population, it follows by a Chernoff bound that
  $\prob{T=1}=e^{-\Omega(n)}$.  We then claim that for all $t\geq 0$,
  $\prob{T=t+1\mid T>t}\leq e^{-c'n}$ for a constant $c'>0,$ which by
  the union bound implies that $\prob{T<e^{cn}}\leq
  e^{cn-c'n}=e^{-\Omega(n)}$ for any constant $c<c'$.

In the initial population, the expected fitness of a $k$-th
individual, $k\in [\lambda]$ is:
$$
E[f(P_0(k))]=\sum_{i=1}^n 0.5 \cdot a_i=0.5 \cdot f^*.
$$
Instead of the fitness values, it will be more convenient here to
consider a deviation from the optimum fitness in individual $j$ of
the current population~$t$. We denote $Z_t^{(j)}:= f^* -
f(P_t(j))$ , for $t\geq
  0,$ $j\in[\lambda],$ and
$Z_t:=\lambda f^* - \sum_{j=1}^{\lambda} f(P_t(j))$. Let $p_j$ be
the probability of selecting
  the $j$-th individual when producing the population in generation
  $t+1$.
  For $f$-monotone selection mechanisms, it holds that
  $\sum_{j=1}^\lambda p_j Z_t^{(j)} \leq Z_t/\lambda.$

Let $P=(x_1,\dots,x_{\lambda})$ be any deterministic population,
and denote the $i$-th bit of the $k$-th individual of $P$ by
$x^{(k,i)}.$ Denote $z_k:= f^* - \sum_{i=1}^{n} a_i x^{(k,i)}$,
$1\le k \le \lambda$, $z(P):=\lambda f^* - \sum_{k=1}^{\lambda}
\linear(x_k)$ and $Z(P):=\lambda f^* - \sum_{j=1}^{\lambda}
f(x_j)$. The expected value of $Z_{t+1}^{(j)}$ for an offspring
$j\in[\lambda]$ is
  \begin{align*}
 \expect{Z_{t+1}^{(j)} \ | \ P_t=P}
 &= f^*- \sum_{i=1}^n \left (a_i x^{(j,i)} (1-\chi/n) + a_i(1-x^{(j,i)}) \chi/n \right)\\
 &= f^*- \sum_{i=1}^n \left(a_i \chi/n + a_i x^{(j,i)} (1-2\chi/n)\right)\\
 &\leq \sum_{k=1}^{\lambda} p_k (f^*- f^* \chi/n) - \sum_{k=1}^{\lambda} p_k (f^*- z_k)(1-2\chi/n)\\
 &= f^*-f^*\chi/n-f^*(1-2\chi/n)+ (1-2\chi/n) \sum_{k=1}^{\lambda} p_k z_k\\
 &\le  f^* \chi/n + (1-2\chi/n)Z(P)/\lambda.
  \end{align*}

If $T>t$ and $Z(P)<\lambda
f^*(1+\varepsilon)/2$, then
  \begin{align*}
    \expect{Z_{t+1}\mid P_t=P}
     &\le \lambda \chi f^* /n + Z(P)\left(1-2\chi/n\right)\\
     & < \lambda\chi f^* /n + \frac{\lambda f^*}{2}(1+\varepsilon)\left(1-2\chi/n\right)
       = \frac{\lambda f^*}{2}(1+\varepsilon)-\varepsilon\lambda\chi f^*/n.
  \end{align*}
  Now $Z_{t+1}^{(1)}, Z_{t+1}^{(2)}, \dots,
  Z_{t+1}^{(\lambda)}$ are non-negative independent random variables,
  each bounded from above by~$f^*$,
  so using the Hoeffding's inequality~\cite{Hoeffding63}
  we obtain
  \begin{align*}
    \prob{ Z_{t+1} \geq \frac{\lambda f^*}{2}(1+\varepsilon)}
   &  \leq   \prob{ Z_{t+1} \geq \expect{Z_{t+1}} + \varepsilon\lambda\chi f^*/n}\\
   & \leq \exp\left( - \frac{2(\varepsilon\lambda\chi
   f^*/n)^2}{\lambda\cdot
       (f^*)^2}\right) = e^{-\Omega(n^{\delta})}. \ \ \qed
  \end{align*}
\end{proof}

\begin{proof}[of Theorem~\ref{thm:GA-on-scaled_linear}]
We apply Theorem~\ref{thm:level-based-theorem} and use
the same partition as the one in the proof of
Theorem~\ref{thm:GA-on-pseudo-boolean-func},
thus the number of levels is also $m = n+1$.

To estimate $s_j$ in (M1), we also consider the probability
of flipping a specific $0$ while keeping the rest of the
string unchanged
$\frac{\chi}{n}\left(1 - \frac{\chi}{n}\right)^{n-1}
  >    \frac{\chi}{n}\left(1 - \frac{1}{n}\right)^{n-1}
  \geq \frac{\chi}{en} =: s_j = s_*
$, and this choice of $s_j$ satisfies~(M1).
%
To satisfy (M2), we pick $p_0 := (1-\chi/n)^n$, \ie the probability
of not flipping any bit position by mutation.

In (M3), we choose $\gamma_0 := \varepsilon/c,$
and since both $\chi$
and $c$ are constants with $c>e^\chi$, so are $\varepsilon$
and $\gamma_0$ with $\gamma_0 \in (0,1)$.
We denote the fitness level of
$\lceil\gamma\lambda\rceil$-ranked individual of any given population
$P$ by~$f_j$. Let $k\ge \lceil
\gamma\lambda\rceil$ be the number of individuals with fitness
 at least $f_j$, and let $s\ge kf_j \ge \lceil
\gamma\lambda\rceil f_j$ be the sum of fitnesses of these $k$
individuals. All the remaining individuals satisfy the inequality
$f(x,c)\le f_j/c$ because $f(x)$ takes only integer values in case
of integer weights~$a_j$. For any $\gamma \leq \gamma_0 =
\varepsilon/c$, the probability of selecting one of the
$k$ individuals in the case of exponential fitness scaling is
\begin{equation*}
\beta(\gamma, P)\ge \frac{s}{(\lambda-k)f_j/c+s}\ge
 \frac{\gamma}{(1-\frac{k}{\lambda})/c+\gamma}\ge
  \frac{\gamma}{(1-\frac{1}{\lambda})/c+\gamma}\ge
  \frac{\gamma c}{1+\gamma c}\ge
  \frac{\gamma c}{1+\varepsilon}.
\end{equation*}
Now note that the lower bound for~$p_0:=(1-\chi/n)^{n}$
as seen in the proof of Theorem~\ref{thm:GA-on-pseudo-boolean-func}
also implies that $p_0\ge
e^{-\chi}/(1+\varepsilon)$ for any constant $\varepsilon>0$
and sufficiently large~$n$. Therefore,
$$
\beta(\gamma,P)p_0
  \ge \frac{\gamma c}{(1+\varepsilon)^2 e^{\chi}}
  =   \gamma(1+\varepsilon)
$$
where the last equality uses $c = (1 + \varepsilon)^3
e^\chi$ from the choice of $\varepsilon$. We have just shown
that (M3) is satisfied for a constant $\delta := \varepsilon$.

Condition~(M4) requires a population size of at least
$\frac{4}{\gamma_0 \delta^2}
\ln \left( \frac{128 m}{\gamma_0 s_* \delta^2 } \right).$
The condition then
holds for any $\lambda \geq \frac{4c}{\varepsilon^3}
\ln \left( \frac{128 (n+1)^2ce}{\varepsilon^3\chi} \right)$.
Since all the conditions are satisfied, application of
Theorem~\ref{thm:level-based-theorem}
implies
\begin{align*}
\expect{T}
    &\leq \frac{8}{\delta^2} \sum_{j=1}^{n} \left(\lambda \ln(3\delta\lambda/2) + \frac{n e}{\gamma_0 \chi}\right)
     =\frac{8}{\varepsilon^2} \sum_{j=1}^{n} \left(\lambda \ln(3\varepsilon\lambda/2) + \frac{n e c}{\varepsilon \chi}\right).
\end{align*}\qed
\end{proof}

\begin{proof}[of Theorem~\ref{thm:GA-on-adf}]
The proof applies Theorem~\ref{thm:level-based-theorem1} with
the partition of the search space into $m=N+1$ levels:
$A_N:=\{x \mid \decomp(x) = \sum_{\ell=1}^N a_\ell =: f^*\}$,
$A_{j}
  := \left\{x \mid \sum_{i=1}^j a_i \leq \decomp(x) < \sum_{i=1}^{j+1} a_i \right\}$
for $j \in \{0\} \cup [N-1]$. For any solution $x$, we say that
it has subsequence $\sigma_\ell$ solved if $f_\ell(x) = 1$.

The partition is such that given $x\in A_{j}$ for any~$j<N$,
among the first $j+1$ subsequences, there must be at least one
subsequence $\sigma_\ell$ that is not solved, thus it suffices
to solve this subsequence while keeping the remaining part of
the string unchanged to produce a search point at a higher
level. The probability of such an event is at least
$ \left(\frac{\chi}{n}\right)^{|\sigma_\ell|}\left(1 - \frac{\chi}{n}\right)^{n-|\sigma_\ell|}
  \geq \left(\frac{1 - c}{n^2 a_1}\right)^r \left(1 - \frac{1}{n}\right)^{n-1}
  \geq \frac{(1-c)^r}{en^{2r} a_1^r} =: s_j = s_*
$, and this choice of $s_j$ satisfies~(M1).
%
To satisfy (M2), we pick $p_0 := (1-\chi/n)^n$, \ie the probability
of not flipping any bit position by mutation.

In (M3), we choose $\gamma_0 := c/4$ and for any $\gamma \leq
\gamma_0$, let $f_\gamma$ be the fitness of the $\lceil\gamma
\lambda\rceil$-ranked individual of any given $P \in
\mathcal{X}^\lambda$. Thus there are at least $k \geq
\lceil\gamma\lambda\rceil \geq \gamma\lambda$ individuals with
fitness at least $f_\gamma$ and let $s \geq k f_\gamma \geq
\gamma\lambda f_\gamma$ be their sum of fitness. Since the
weights~$a_i$ are all integers, we can pessimistically assume that
individuals with fitness less than $f_\gamma$ have fitness
$f_\gamma - 1$, therefore 
\begin{align*}
  \beta(\gamma, P)
    &\geq \frac{s}{s + (\lambda - k)(f_\gamma - 1)}
     \geq \frac{s}{s + (\lambda - \gamma\lambda)(f_\gamma - 1)} \\ 
    &\geq \frac{\gamma\lambda f_\gamma}{\gamma\lambda f_\gamma + (\lambda - \gamma\lambda)(f_\gamma - 1)}
     =    \frac{\gamma}{1 - (1 - \gamma)/f_\gamma} \\ 
    &\geq \frac{\gamma}{1 - (1 - c/4)/f^*}
     \geq \gamma e^{(1 - c/4)/f^*}, 
\end{align*}
where in the last line we apply
the inequality $e^{-x}\ge 1-x.$
We reuse the argument from the proof of
Theorem~\ref{thm:GA-on-pseudo-boolean-func}
to bound $p_0$ from below by $e^{-\chi/(1-\varepsilon)}$
for any constant $\varepsilon \in (0,1)$ and sufficiently large
$n$, and
conclude that
\begin{equation*} 
  \beta(\gamma, P) p_0
    \ge \gamma e^{(1 - c/4)/f^*} e^{-\chi/(1 - \varepsilon)}
    \ge \gamma\left(1+\frac{1 - c/4 -\chi f^*/(1-\varepsilon)}{f^*}\right).
\end{equation*}
Since $\chi f^* \leq \chi N a_1 \leq \chi n a_1 = 1 - c$, choosing
$\varepsilon := 1 - \frac{1-c}{1-c/2} \in (0,1)$
implies $\chi f^*/(1 - \varepsilon) \leq 1 - c/2$. Condition (M3)
then holds for $\delta:=c/(4 n a_1)$ because
\begin{equation*} 
\beta(\gamma, P) p_0
  \geq \gamma\left(1+\frac{1 - c/4 - (1 - c/2)}{f^*}\right)
  \geq \gamma\left(1+\frac{c}{4 n a_1}\right).
\end{equation*}

To verify condition~(M4'), we assume $C=1$ and note that
$$
               \frac{8}{\gamma_0\delta^2} \log\left(\frac{C m}{\delta} \left(\log \lambda +\frac{1}{\gamma_0 s_* \lambda}\right) \right)
$$
$$
={\cal O}\left(n^2 a_1^2 \log\left({\cal O}(n^2 a_1) \left(\log(n^K) +\frac{e n^{2r} a_1^r}{(c/4)(1-c)^r\lambda} \right)\right)\right)
$$
$$
={\cal O}(n^2 a_1^2 r\log(n a_1)),
$$
so~(M4) holds if $c'$ is large enough.

By Theorem~\ref{thm:level-based-theorem},
$\expect{T}=\bigO{n^2 a_1 \lambda \log \lambda + n^{2r+2} a_1^{r+1}/(1-c)^r}=
\bigO{n^2 a_1 \lambda \log(n a_1) + n^{2r+2} a_1^{r+1}(1-c)^{-r}}$.\qed
\end{proof}

\end{document}